\documentclass[11pt]{amsart}

\usepackage{verbatim, amssymb,hyperref}
\usepackage{accents}
\usepackage{color}
\usepackage{soul}
\usepackage{bbm}
\usepackage{graphicx}
\usepackage{amsmath}

\newcommand{\err}{\mathrm{err}}

\newcommand{\fd}{\mathfrak{d}}

\newcommand{\cD}{\mathcal D}
\newcommand{\cE}{\mathcal E}

\newcommand{\cH}{\mathcal H}

\newcommand{\cL}{\mathcal L}
\newcommand{\cR}{\mathcal R}

\newcommand{\cO}{\mathcal O}

\newcommand{\dd}{\mathrm{d}}

\newcommand{\IW}{\mathbb W}

\newcommand{\cW}{\mathcal W}

\newcommand{\vecn}{\mathfrak n}

\newcommand{\II}{\mathbb I}

\renewcommand{\IJ}{\mathbb J}

\newcommand{\1}{1\hspace{-0.098cm}\mathrm{l}}

\newcommand{\N}{{\mathbb N}}
\newcommand{\IA}{{\mathbb A}}

\newcommand{\R}{{\mathbb R}}

\newcommand{\cN}{{\mathcal N}}

\newcommand{\ID}{\mathbb D}

\theoremstyle{plain}
\newtheorem{theorem}{Theorem}[section]
\newtheorem{prop}[theorem]{Proposition}

\newtheorem{defi}[theorem]{Definition}

\theoremstyle{definition}
\newtheorem{rem}[theorem]{Remark}
\newtheorem{exa}[theorem]{Example}

\makeatletter
\@namedef{subjclassname@2020}{%
	\textup{2020} Mathematics Subject Classification}
\makeatother

\begin{document}

\title[On the existence of optimal shallow feedforward ReLU networks]%
{On the existence of optimal shallow \\ feedforward networks with ReLU activation}

\author[]%[Dereich]
{Steffen Dereich}
\address{Steffen Dereich\\
	Institute for Mathematical Stochastics\\
	Faculty of Mathematics and Computer Science\\
	University of M\"{u}nster, Germany}
\email{steffen.dereich@uni-muenster.de}

\author[]%[Kassing]
{Sebastian Kassing}
\address{Sebastian Kassing\\
	Faculty of Mathematics\\
	University of Bielefeld, Germany}
\email{skassing@math.uni-bielefeld.de}

\keywords{Neural networks, shallow networks, best approximation, ReLU activation, approximatively compact}
\subjclass[2020]{Primary 68T07; Secondary 68T05, 41A50}

\begin{abstract} 
	We prove existence of global minima in the loss landscape for the approximation
	of continuous target functions using shallow feedforward artificial neural networks with ReLU activation. This property is one of the fundamental artifacts separating ReLU from other commonly used activation functions.
	We propose a kind of closure of the search space so that in  the extended space minimizers exist.  In a second step, we show under mild assumptions that the newly added functions in the extension perform worse than appropriate representable ReLU networks. This then implies that the optimal response in the extended target space is indeed the response of a ReLU network.
\end{abstract}

\maketitle

\section{Introduction} 
Modern machine learning algorithms are commonly based on the optimization of artificial neural networks (ANNs) through gradient based algorithms. The overwhelming success of these methods in practical applications has encouraged many scientists to build the mathematical foundations of machine learning and, in particular, to identify universal structures in the training dynamics that might provide an explanation for the mind-blowing observations practitioners make. One key component of ANNs is the activation function. Among the various activation functions that have been proposed, the rectified linear unit (ReLU), which is defined as the maximum between zero and the input value, has emerged as the most widely used and most effective activation function. There are several reasons why ReLU has become such a popular choice, e.g., it is easy to implement, computational efficient and overcomes the vanishing gradient problem, which is a common issue with other activation functions when training ANNs. In this work, we point out and prove a more subtle feature of the ReLU function that separates ReLU from several other common activation functions and might be one of the key reasons for its popularity in practice: the existence of global minima in the optimization landscape.

A popular line of research studies the optimization procedure (also called \emph{training}) for ANNs using gradient descent (GD) type methods. Since the \emph{loss/error function} in a typical machine learning optimization task is non-linear, non-convex and even non-coercive it remains an open problem to rigorously prove (or disprove) convergence of GD even in the simplest scenario of optimizing a shallow ANN, i.e., an ANN with only one hidden layer. Existing theoretical convergence results often assume the process to stay bounded, i.e., for every realization there exists a compact set such that the process does not leave this set during training, see e.g. \cite{bolte2007lojasiewicz, Eberle2021arXiv, jentzen2021existence}
for results concerning gradient flows, 
\cite{absil2005convergence, attouch2009convergence} for results concerning deterministic gradient methods,
\cite{TADIC2015convergence,MR4056927, mertikopoulos2020sure, dereich2021convergence} for results concerning stochastic gradient methods
and
\cite{dereich2021cooling} for results concerning gradient based diffusion processes. Many results go back to classical works by \L ojasiewicz concerning gradient inequalities for analytic target functions and direct consequences for the convergence of gradient flow trajectories under the assumption of staying local \cite{lojasiewicz1963propriete,lojasiewicz1965ensembles,lojasiewicz1984trajectoires}.

In this context, it seems natural to ask for the existence of ANNs that solve the minimization task within the search space. More explicitly, if there does not exist a global minimum in the optimization landscape then every sequence that approaches the minimal loss value diverges to infinity. This might lead to slow convergence or even rule out convergence of the loss value, which is the property that practitioners are most interested in. Therefore, it seems reasonable to choose a network architecture, activation function and loss function such that there exist global optima in the optimization landscape. 

Overparametrized ANNs in the setting of empirical risk minimization (more neurons in the hidden layer than data points to fit) are often able to perfectly interpolate the finitely many data points such that there exists a network configuration achieving zero error and, thus, a global minimum in the search space. However, for general measures $\mu$ not necessarily consisting of a finite number of Dirac measures, the literature on the existence of global minima is very limited. There exist positive results for the approximation of functions in the space $L_p([0,1]^d, \|\cdot\|_p)$ with shallow feedforward ANNs using \emph{heavyside activation} \cite{kainen2003best}, the approximation of \emph{one-dimensional} Lipschitz continuous target functions with shallow feedforward ANNs using ReLU activation and the standard mean square error \cite{jentzen2021existence}, and the approximation of multi-dimensional, continuous target functions with shallow \emph{residual} ANNs using ReLU activation, see
\cite{DJK22}. Conversely, for several common (smooth) activations such as the standard logistic activation, softplus, arctan, hyperbolic tangent and softsign there, generally, do not exist minimizers in the optimization landscape for smooth target functions (or even polynomials), see \cite{petersen2021topological} and \cite{gallonjentzenlindner2022blowup}. 
This phenomenon can also be observed in empirical risk minimization for the hyperbolic tangent activation. As shown in  \cite{lim2022best},  in the underparametrized setting,
%We also point to \cite{lim2022best} where the existence of minimizers in the underparametrized regime of empirical risk minimization is considered. For the hyperbolic tangent activation it is show that 
there exist input data such that for all output data from a set of positive Lebesgue measure there does not exist  minimizers in the optimization landscape. %Conversely, while there exist examples of input-output data such that there does not exist minimal ReLU networks, 
It remains an open problem whether for ReLU activation this phenomenon prevails.

In this article, we prove, for the first time, existence results for \emph{shallow feedforward ReLU ANNs} with multi-dimensional input space. Interestingly, minimizers exist under very mild assumptions on the optimization problem. This existence property indicates the robustness of ReLU activation and may be a reason for its success in practical applications.
For the proof we proceed as follows. First,  we  show existence of minimizers in an extended target space that comprises of the representable responses of ANNs and additional discontinuous generalized responses. Note that for many activation functions (including ReLU) the set of \emph{realization/response functions} is not closed for an appropriate metric, see also \cite{girosi1990networks}.  Second, we show that the additional discontinuous responses perform worse than representable ones under mild conditions on the optimization problem. Compared to~\cite{DJK22}, where residual networks are treated, the situation is more complex in classical feedforward ReLU network  as treated here. This is caused by the more sophisticated structure of the extended search space, see Definition~\ref{def:genres}.

We present a special case of our main result in the situation where we focus on approximation of continuous target functions with shallow ANNs using ReLU activation under $L_p$-loss.  

\begin{theorem}
	\label{thm:main_simplified0}
	Let 
	$ d_{\mathrm{in}}, d \in \N $, $p > 1$ and
	$ 
	\fd 
	=
	( d_{\mathrm{in}} + 2 )
	d
	+ 1
	$. Let 
	$ f \colon \R^{ d_{\mathrm{in}} } \to \R $ 
	and 
	$ h \colon \R^{ d_{\mathrm{in}} } \to [0,\infty) $
	be continuous functions and
	assume that 
	$
	h^{ - 1 }( (0,\infty) )
	$
	is a bounded and convex set.
	For every 
	$ \theta = ( \theta_1, \dots, \theta_{ \fd } ) \in \R^{ \fd } $
	let 
		$
	\mathrm{err} 
	\colon 
	\R^{ \fd }
	\to 
	\R
	$
	be given by 
	$
	\textstyle
	\mathrm{err}( \theta )
	=
	\int_{ 
		\R^{ 
			d_{\mathrm{in}} 
		} 
	}
	|
	f(x)
	-
	\mathfrak{N}_{ \theta }( x )
	|^p
	h(x)
	\,
	dx,
	$
	where
	\begin{equation*}
		\begin{split}
			\textstyle 
			\mathfrak{N}_{ \theta }( x ) 
			& = 
			\textstyle 
			\theta_{ \fd } 
			+
			\bigl(
			\sum_{ j = 1 }^{ 
				d
			}
			\theta_{ 
				(d_{\mathrm{in}}+1) d
				 + j 
			}
			\max\{
			\theta_{ d_{\mathrm{in}} d + j }
			+
			\sum_{ i = 1 }^{d_{\mathrm{in}}}
			\theta_{ (j-1) d_{\mathrm{in}} + i } 
			x_i
			, 0
			\}
			\bigr)
			.
		\end{split}
	\end{equation*}
	Then there exists $ \theta \in \R^{ \fd } $
	such that 
	$
	\mathrm{err}( \theta )
	=
	\inf_{ \vartheta \in \R^{ \fd } }
	\mathrm{err}( \vartheta )
	$. 
\end{theorem}
Theorem~\ref{thm:main_simplified0} is a special case of the more general Theorem~\ref{thm:main2}, which treats a broader class of loss functions and measures. 

Let us introduce the central objects and notations of this article. In the following, we represent ANNs in a more structurized way. We consider networks with $d_\mathrm{in}$-dimensional input space and one hidden layer consisting of $d$ neurons that apply ReLU activation, i.e. $(x)^+=\max(x,0)$. We describe the weights of the ANN by a matrix   $W^1=(w_{j,i}^1)_{j=1,\dots,d,i=1,\dots,d_\mathrm{in}}$ and a row vector $W^2=(w_{1}^2,\dots, w_d^2)$, and the biases by a column vector $b^1=(b^1_i)_{i=1,\dots,d}$ and a scalar $b^2$.
Moreover, for $j=1, \dots,d$, we write $w_j^1=(w_{j,1}^1, \dots, w_{j,d_{\mathrm{in}}}^1)^\dagger$, where $a^\dagger$ denotes the transpose of a vector or a matrix $a$. We let
$$
\IW=(W^1,b^1, W^2, b^2) \in \R^{d\times d_{\mathrm{in}}}\times \R^{d}\times \R^{1\times d}\times \R 
=:\cW_d
$$
and call $\IW$ a \emph{network configuration} and $\cW_d$ the \emph{parametrization class}.
We often refer to a configuration of a neural network as the (\emph{neural}) \emph{network}~$\IW$. A configuration $\IW \in \cW_d$ describes a function $\mathfrak N^\IW: \R^{d_{\mathrm{in}}} \to \R$ via
\begin{align} \label{eq:response}
	\mathfrak N^\IW(x)= \sum_{j=1}^{d} w_{j}^2 \bigl( w_{j}^1\cdot x+b_j^1\bigr)^++b^2,
\end{align}
where $\cdot$ denotes the scalar product on $\R^{d_{\mathrm{in}}}$.
We call $\mathfrak N^\IW$ \emph{realization function} or \emph{response} of the network~$\IW$. We allow as parameter $d$ all values from $\N_0:=\{0,1,\dots\}$, where a response of a network with zero neurons is a constant function (by definition).
For an introduction into general neural networks with possibly multiple hidden layers see e.g.~\cite{petersen2021topological}.

Note that, in general, the response of a network is a continuous, piecewise affine function from  $\R^{d_\mathrm{in}}$ to $\R$. We conceive $\IW \mapsto \mathfrak N^{\IW}$ as a parametrization of a class of potential response functions  $\{\mathfrak N^{\IW}: \IW \in \cW_d\}$ in a minimization problem.  
More explicitly, let $\mu$ be a finite measure on
the Borel sets of $\R^{d_{\mathrm{in}}}$, let $\ID=\mathrm{supp}(\mu)$  and $\mathcal L:\ID\times \R \to \R_+$ be a product-measurable function, the \emph{loss function}. 
We aim to minimize the error
	$$
	\mathrm{err}^{\cL}(\IW)=\int_\ID \mathcal L(x,\mathfrak{N}^\IW(x))\,\dd \mu(x)
	$$
over all $\IW \in\cW_d$ for a given $d\in\N_0$
	and let
	\begin{align}\label{eq_min}
	\mathrm{err}^{\cL}_{d}=\inf_{\IW \in \cW_d}\mathrm{err}^{\cL}(\IW)
	\end{align}
be the minimal error for the optimization task when using a neural network with $d$ neurons in the hidden layer.

The aim of this work is to give sufficient conditions on the loss function $\cL$ and the measure $\mu$ that guarantee existence of a network $\IW \in \cW_d$ with $\err^{\cL}(\IW)=\err_d^\cL$. We stress that if there does not exist a neural network $\IW \in \cW_d$ satisfying $\err^\cL(\IW) = \err_d^\cL$ then every sequence $(\IW_n)_{n \in \N}\subset \cW_d$ of networks satisfying $\lim_{n \to \infty} \err^\cL(\IW_n) = \err_d^\cL$ diverges to infinity. Our framework includes the approximation of a continuous function using the $L_p$-loss, see Theorem~\ref{thm:main_simplified0}.

For a general introduction into best approximators in normed spaces we refer the reader to~\cite{singerbest}. A good literature review regarding the loss landscape in neural network training can be found in \cite{E2020towards}. For statements about the existence of non-optimal local minima in the training of (shallow) networks we refer the reader to \cite{swirszcz2016local, safran2018spurious, Venturi2019Spurious} and \cite{christof2022omnipresence}.

We state the main result of this article.
	
	\begin{theorem} \label{thm:main2}
		Suppose that $\ID=\mathrm{supp}(\mu)$ is compact and that $\mu$ has a continuous Lebesgue density $h:\R^{d_\mathrm{in}}\to\R_+$. Assume that for every hyperplane $H$ that intersects the interior of the convex hull of $\ID$, there exists an $x\in H$ with $h(x)>0$. Moreover, assume that the loss function $\cL: \ID \times \R \to \R_{+}$ satisfies the following assumptions:
		\begin{enumerate}
			\item[(i)]
			(Continuity in the first argument)  For every $y\in \R$, $\ID\ni x\mapsto \cL(x,y)$ is continuous. 
			\item[(ii)] (Strict convexity in the second argument) For all $x \in \ID$, $y \mapsto \cL(x,y)$ is strictly convex and attains its minimum.  
		\end{enumerate}
		Then, for every $d\in\N_0$, there exists an optimal  network $\IW\in\cW_d$ with $\mathrm{err}^{\cL}(\IW)=\mathrm{err}^{\cL}_d$.
	\end{theorem}

Theorem~\ref{thm:main2} is an immediate consequence of Proposition~\ref{prop:minimal} below. We stress that the statement of Proposition~\ref{prop:minimal} is stronger in the sense that it even shows that in many situations the newly added functions to the extended target space perform strictly worse than the representable responses.

\begin{exa}[Regression problem]
	Let $\mu$ be as in Theorem~\ref{thm:main2} and  suppose that $f:\R^{d_\mathrm{in}}\to \R$ is a continuous and $L:\R\to\R_+$ a strictly convex function that attains its minimum. Then $\cL:\R^{d_\mathrm{in}}\times \R\to \R_+$ given by
	$$
	\cL(x,y)=L(y-f(x))
	$$
	satisfies the assumptions of the latter theorem and, thus, the infimum 
		$$
	\inf _{\IW\in\cW_d} \int L(\mathfrak N^\IW(x)-f(x)) \,\dd \mu(x)
	$$
	is attained for a network  $\IW\in\cW_d$.
\end{exa}

\section{Generalized response of neural networks}
We will work with more intuitive  geometric descriptions of realization functions of networks $\IW\in \cW_d$ as introduced in \cite{DJK22}.
We call a network $\IW\in\cW_d$ \emph{non-degenerate} iff for all $j=1,\dots,d$ we have $w_j^1\not=0$. For a non-degenerate network $\IW$, we say that the neuron $j\in\{1,\dots, d\}$  has
\begin{itemize} 
	\item  \emph{normal}  ${\displaystyle \vecn_j=\frac {1  }{|w_j^1|} w_j^1\in \cO:=\{x\in\R^{d_\mathrm{in}}: |x|=1\}}$,
	\item \emph{offset} $o_j=-\frac {1}{|w_j^1|} b_j^1\in\R$ and
	\item \emph{kink} $\Delta_j=|w_j^1| w_j^2\in\R$. 
	%\item kink $\Delta_j= \orient(w_j^1) w_j^2 w_j^1$,
\end{itemize}
Moreover, we call $\mathfrak b=b^2$ the \emph{bias} of $\IW$. We call $(\vecn, o, \Delta,  \mathfrak b)$ with $\vecn = (\vecn_1, \dots, \vecn_{d})\in \cO^d$, $o=(o_1, \dots, o_d)\in \R^d$, $\Delta=(\Delta_1, \dots, \Delta_d)\in\R^d$ and $\mathfrak b \in \R$ the \emph{effective tuple of $\IW$} and write $\cE_{d}$ for the set of all effective tuples using $d$ ReLU neurons.

First we note that the response of a non-degenerate network $\IW$ can be represented in terms of its effective tuple: one has, for $x\in\R^{d_\mathrm{in}}$,% using that $\sgn(w_j^2)=\orient(\Delta_j)$
\begin{align*}
	\mathfrak N^\IW(x)&=\mathfrak b+\sum_{j=1}^{d} w^2_j(w^1_j \cdot x+b_j^1)^+=\mathfrak b+\sum_{j=1}^{d}\Delta_j \Bigl(\frac 1{|w_j^1|} w^1_j \cdot x+\frac 1{|w_j^1|} b_j^1\Bigr)^+\\
	&=\mathfrak b+\sum_{j=1}^{d} \Delta_j \bigl( \vecn_j \cdot x- o_j\bigr)^+.
\end{align*} 
With slight misuse of notation we also write
$$
\mathfrak{N}^{\vecn, o, \Delta,\mathfrak b}:\R^{d_{\mathrm{in}}}\to \R,\, x\mapsto \mathfrak b+\sum_{j=1}^{d} \Delta_j \bigl( \vecn_j \cdot x- o_j\bigr)^+
$$
and $\err^{\cL}(\vecn,o, \Delta, \mathfrak b)= \int \cL(x, \mathfrak N^{\vecn, o, \Delta, \mathfrak b}(x)) \, \dd \mu(x).$ 
Although the tuple $(\vecn,o,\Delta,\mathfrak b)$ does not  uniquely describe a neural network, it describes a response function uniquely and thus we will speak of the neural network with effective tuple  $(\vecn,o,\Delta,\mathfrak b)$.

We stress that the response of a degenerate network $\IW$ can also be described as response associated to an effective tuple. Indeed, for every $j\in \{1, \dots, d\}$ with $w_j^1=0$  the respective neuron has a constant contribution  $w_j^2(b_j^1)^+$. Now, one can choose an arbitrary normal $\vecn_j$ and offset $o_j$, set the kink equal to zero ($\Delta_j=0$) and add the constant $w_j^2(b_j^1)^+$ to the bias $\mathfrak b$. Repeating this procedure for every such neuron we get an effective tuple $(\vecn, o, \Delta, \mathfrak b) \in \cE_d$ that satisfies $\mathfrak N^{\vecn, o, \Delta, \mathfrak b} = \mathfrak N^{\IW}$. Conversely, for every effective tuple $(\vecn, o, \Delta,\mathfrak b) \in \cE_d$,
$\mathfrak N^{\vecn,o,\Delta,\mathfrak b}$ is the response of an appropriate network $\IW \in \cW_d$. In fact one can choose $b^2=\mathfrak b$ and, for $j=1,\dots, d$,
 $w_j^1=\vecn_j$, $b_j^1=-o_j$ and $w_j^2=\Delta_j$ such that for all $x \in \R^{d_{\mathrm{in}}}$
$$
w_j^2(w_j^1 \cdot x + b_j^1)^+ = \Delta_j (\vecn_j \cdot x-o_j)^+.
$$
This entails that
$$
\mathrm{err}^{\cL}_d = \inf_{(\vecn,o,\Delta,\mathfrak b)\in\cE_d}\int \cL(x,\mathfrak N^{\vecn,o,\Delta,\mathfrak b}(x))\, \dd\mu(x)
$$ 
and the infimum is attained iff there is a network $\IW\in\cW_d$ for which the infimum in~(\ref{eq_min}) is attained.
For an effective tuple $(\vecn, o, \Delta, \mathfrak b) \in \cE_d$,  we say that the $j$th ReLU neuron has the \emph{breakline}
$$
H_j=\bigl\{x\in\R^{d_{\mathrm{in}}}:\vecn_j \cdot x =  o_j\bigr\}
$$
and we call 
$$
A_j=\{x \in \R^{d_{\mathrm{in}}}: \vecn_j \cdot x>o_j\}
$$
the \emph{domain of activity} of the $j$th ReLU neuron.
By construction, we have 
$$
\mathfrak{N}^{\vecn,o,\Delta,\mathfrak b}(x)= \mathfrak b+\sum_{j=1}^{d} \1_{A_j}(x)  \bigl(\Delta_j( \vecn_j \cdot x- o_j)\bigr).
$$
Outside the breaklines, the function $\mathfrak{N}^{\vecn,o,\Delta,\mathfrak b}$ is differentiable with
$$
D\mathfrak{N}^{\vecn,o,\Delta,\mathfrak a}(x)= \sum_{j=1}^{d} \1_{A_j}(x)   \Delta_j  \vecn_j.
$$
Note that for each summand $j=1, \dots, d$ along the breakline the difference of the differential on $A_j$ and $\overline{A_j}^c$ equals $\Delta_j \vecn_j$ (which is also true for the response function $\mathfrak N^\IW$ provided that it is differentiable in the reference points and there does not exist a second neuron having the same breakline $H_j$).

Next, we introduce the class of generalized network responses. This class has the advantage that under quite mild assumptions minimizers can be found by applying compactness arguments.

\begin{defi} \label{def:genres} We call a function $\cR:\R^{d_{\mathrm{in}}} \to \R$  \emph{generalized response} if it admits the following representation: there are
$K\in\N_0$,  a tuple of open half-spaces $\mathbf A=(A_1,\dots,A_K)$ of $\R^{d_{\mathrm{in}}}$ with pairwise distinct boundaries $\partial A_1,\dots,\partial A_K$, $\mathbf m =(m_0,\dots,m_K)\in\{0,1\}\times \{1,2\}^K$, an affine mapping $\mathfrak a:\R^{d_\mathrm{in}}\to\R$, vectors $\delta_1, \dots, \delta_K \in \R^{d_{\mathrm{in}}}$ and reals $\mathfrak b_1, \dots, \mathfrak b_K \in \R$ such that  for all $x \in \R^{d_{\mathrm{in}}}$
	\begin{align}\label{eq84529}
		\cR(x)= \mathfrak a(x)+ \sum_{k=1}^K \1_{A_k}(x) \bigl(\delta_{k}\cdot x+\mathfrak b_{k}\bigr),
	\end{align}
   for all $k=1,\dots,K$ with $m_k=1$,
	$$
	\partial A_k \subset \{x \in \R^{d_{\mathrm{in}}}:\delta_k\cdot x+\mathfrak b_k=0\}
	$$
	and (at least) one of the following properties is true
	\begin{enumerate} \item[(a)] $m_0=1$, 
	\item[(b)] $(\mathfrak n_k: k\in\{1,\dots,K\}\text{ with }m_k=2)$ is linearly dependent or
		\item[(c)] $\mathfrak a$ is a constant function.
\end{enumerate}
The minimal number $m_0+\ldots+m_K$ that can be achieved is called \emph{dimension} of the generalized response.
\end{defi}

We call $A_1,\dots,A_K$ \emph{active half-spaces} of the response, $m_1,\dots,m_K$ the \emph{multiplicities} of the half-spaces $A_1,\dots,A_K$ or  summands.
For $d \in \N_0$, we denote by $\mathfrak R_d$ the space of all generalized responses of dimension at most $d$.
Moreover, we call a response $\cR\in \mathfrak R_d$ \emph{strict at dimension $d$} if it is of dimension $d-1$ or lower or if it is discontinuous. Denote by  $\mathfrak R_d^\mathrm{strict}$ the set of strict responses at dimension $d$.
 Moreover, we call a response \emph{representable} if it is continuous or, equivalently, if the multiplicities $m_1,\dots,m_K$ can be chosen to be one.

\begin{rem}
Strictly speaking, a representable response of dimension $d$ is not necessarily the response of a network with $d$ neurons in the hidden layer since in the case with $m_0=1$ we might need two ReLU neurons (instead of one) to generate the linear component of $\mathfrak a$. However, for every compact set $K\subset\R^{d_\mathrm{in}}$ and every representable response $\cR\in\mathfrak R_d$ we can find a network with $d$ neurons in the hidden layer whose response agrees on $K$ with $\cR$. Consequently, for \emph{compactly} supported measures $\mu$, the subset of representable generalized responses can all be realized on the relevant domain by appropriate shallow networks.
\end{rem}

When analyzing the minimization problem over the class of generalized responses, we can impose weaker assumptions than in the main theorem (Theorem~\ref{thm:main2}). We will use the following concepts.

\begin{defi}
	\begin{enumerate} \item[(i)] An element $x$ of a hyperplane $H\subset \R^{d_\mathrm{in}}$ is called \emph{$H$-regular} if $x\in \mathrm{supp}\, \mu|_A$ and $x\in \mathrm{supp}\, \mu|_{\overline A^c}$, where $A$ is an open half-space with $\partial A=H$.
		\item[(ii)] A measure $\mu$ on $\R^{d_\mathrm{in}}$ is called \emph{nice} if %{\blue it is compactly supported, }
		all hyperplanes have $\mu$-measure zero and if for every open half-space $A$ with $\mu(A),\mu(\overline A^c)>0$ the set of $\partial A$-regular points cannot be covered by finitely many hyperplanes different from $\partial A$.
	\end{enumerate} 
\end{defi}

\begin{prop} \label{prop:genres} 
	Assume that $\mu$ is a  nice measure on~$\R^{d_\mathrm{in}}$  
and that the loss function $\cL: \ID \times \R \to \R_{+}$ is measurable and satisfies the following assumptions:
		\begin{enumerate}
			\item [(i)] (Lower-semincontinuity in the second argument) For all $x \in \ID$ and $y \in \R$, we have
			$$
			\liminf\limits_{y' \to y} \cL(x,y') \ge \cL(x,y).
			$$ 
			\item [(ii)] (Unbounded in the second argument) For all $x \in \ID$, we have
			$$
			\lim\limits_{|y| \to \infty} \cL(x,y) = \infty.
			$$
		\end{enumerate}
		Let $d\in\N_0$ with  $\mathrm{err}_d^{\cL}<\infty$. Then there exists a $\cR\in\mathfrak R_d$ with
		$$
		\int \cL(x, \cR(x))\,\dd\mu(x) =\overline{\mathrm{err}}^\cL_d:=  \inf_{\tilde \cR\in\mathfrak R_d} \int \cL(x, \tilde\cR(x))\,\dd\mu(x).
		$$
		Furthermore, for $d \ge 1$ the infimum
		$$
		\inf_{\tilde \cR \in \mathfrak R_d^{\mathrm{strict}}}  \int \cL(x, \tilde\cR(x))\,\dd\mu(x)
		$$
		 is attained on $ \mathfrak R_d^{\mathrm{strict}}$.
\end{prop}

\begin{proof}
			Let $(\cR^{(n)})_{n \in \N}$ be a sequence of generalized responses of at most dimension $d$ that satisfy
		$$
		\lim\limits_{n \to \infty} \int \cL(x,\cR^{(n)}(x)) \, \dd \mu (x) = \overline{\mathrm{err}}_d^{\cL}.
		$$
		We use the representations as in~(\ref{eq84529}) and write
		$$
		\cR^{(n)}(x)=\mathfrak a^{(n)}(x)+\sum_{k=1}^{K_n} \1_{A_k^{(n)}}\bigl(\delta_k^{(n)}\cdot x+\mathfrak b_k^{(n)}\bigr).
		$$
		Moreover, denote by $\mathfrak n_k^{(n)}\in \cO$ and $o_k^{(n)}\in\R$ the quantities with  $A_k^{(n)}=\{x\in\R^{d_\mathrm{in}}: \mathfrak n_k^{(n)}\cdot x > o_k^{(n)}\}$.
		
		\underline{1. step:} Deriving a limit admitting a representation~(\ref{eq84529}).
		
		We choose a  subsequence $(n_l)_{l \in \N}$ along which  always
		 the $K$-number and the multiplicities are the same and so that always the same  case (a), (b) or (c) enters. Moreover, we assume that 		for each $k=1,\dots,K$, $(\vecn_k^{(n_l)})_{l\in\N}$ converges in $\cO$ to  $\vecn_k$ and $(o_k^{(n_l)})_{l\in\N}$ in $\R\cup\{\pm\infty\}$ to $o_k$. For ease of notation we will assume that this is the case for the full sequence.
		
		We call
		$$
		A_k=\{x\in\R^{d_\mathrm{in}}:\vecn_k\cdot x>o_k\}
		$$ 
		the asymptotic active area of the $k$th term and let $H_k=\partial A_k$.		
		Let   $\IJ$ denote the collection of all subsets $J\subset \{1,\dots,K\}$ for which the set
		$$
		A_J=  \bigcap_{j\in J} A_j \cap \bigcap_{j\in J^c} \overline A_j^c
		$$
		satisfies $\mu(A_J)>0$. We note that the sets $(A_J:J\in\IJ)$ are non-empty, open and pairwise disjoint and their union has full $\mu$-measure
		since 
		$$
		\mu\Bigl(\R^{d_\mathrm{in}} \backslash \bigcup_{J \subset \{1, \dots, K\}} A_J \Bigr) \le \sum\limits_{j=1}^K \mu\bigl(  H_j \bigr)=0.
		$$ 
		%We note that a degenerate neuron $j$ would be either in all or none sets from $J\in\IJ$.
		Moreover, for every $J\in\IJ$ and every compact set $B$ with $B\subset A_J$ one has from a $B$-dependent~$n$ onwards that the generalized response $\cR^{(n)}$ satisfies for all $x\in B$ that
		$$
		\cR^{(n)}(x)= \cD_J ^{(n)} \cdot x +\beta_J^{(n)},
		$$
		where  
		$$
		\cD_J^{(n)}:= {\mathfrak a'}^{(n)}+\sum_{j\in J} \delta_j^{(n)} \text{ \ and \ } \beta_J^{(n)}:=\mathfrak a^{(n)}(0) +\sum_{j\in J} \mathfrak b_j^{(n)}.
		$$
		
		Let $J\in\IJ$. Next, we show that along an appropriate subsequence, we have convergence of $(\cD_J^{(n)})_{n \in \N}$  in $\R^{d_\mathrm{in}}$. First assume that along a subsequence  one has that $(|\cD_J^{(n)}|)_{n \in \N}$ converges to $\infty$. For ease of notation we  assume without loss of generality that one has  $|\cD_J^{(n)}|\to\infty$. We 
		let
		$$
		\cH_J^{(n)}= \{x\in\R^{d_\mathrm{in}} : \cD_J^{(n)}\cdot x+\beta_J=0 \}.
		$$ 
		For every $n$ with $\cD_J^{(n)}\not=0$, $\cH_J^{(n)}$ is a hyperplane which can be parametrized by taking a normal and the respective offset. As above we can argue that along an appropriate subsequence (which is again assumed to be the whole sequence) one has convergence of the normals in $\cO$ and of the offsets in $\bar \R$. We denote by $\cH_J$ the hyperplane being associated to the limiting normal and offset (which is assumed to be the empty set in the case where the offsets do not converge in $\R$). Since the norm of the gradient $\cD_J^{(n)}$ tends to infinity we get that for every $x\in A_J\backslash\cH_J$ one has
		$|\cR^{(n)}(x)|\to \infty$ and, hence, $\cL(x,\cR^{(n)}(x))\to \infty$. Consequently, Fatou implies that
		\begin{align*}
			\liminf\limits_{n \to \infty} \int_{A_J\backslash \cH_J}  \cL(x,\cR^{(n)}(x)) \, \dd \mu(x) &\ge \int_{A_J\backslash \cH_J} \liminf\limits_{n \to \infty} \cL(x,\cR^{(n)}(x)) \, \dd \mu(x) = \infty 
		\end{align*}
		contradicting the asymptotic optimality of $(\cR^{(n)})_{n \in \N}$. 
		We showed that the sequence $(\cD_J^{(n)})_{n \in \N}$ is precompact and by switching to an appropriate subsequence we can guarantee that  the limit $\cD_J=\lim_{n\to\infty}\cD_J^{(n)}$ exists.
		
	Similarly, we show that along an appropriate subsequence,  $(\beta_J^{(n)})_{n \in \N}$ converges to a value $\beta_J\in\R$. Suppose this were not the case, then there were a subsequence along which $|\beta_J^{(n)}|\to\infty$.  Again we assume for ease of notation that this were the case along the full sequence. Then, for every $x\in A_J$, one has  that $|\cR^{(n)}(x)|\to\infty$ and we argue as above to show that this would contradict the optimality of $(\cR^{(n)})_{n \in \N}$. 
	Consequently, we have on a compact set $B\subset A_J$ uniform convergence 
	\begin{align}\label{eq87314}
		\lim_{n\to\infty} \cR^{(n)}(x)= \cD_J \cdot x+\beta_J.
	\end{align}
	
	Since $\bigcup_{J\in\IJ} A_J$ has full $\mu$-measure we get  with the lower semicontinuity of $\cL$ in the second argument and Fatou's lemma that for every measurable function $\cR:\R^{d_{\mathrm{in}}}\to\R$ satisfying for each $J\in\IJ$ and $x\in A_J$ that
	\begin{align} \label{eq:DefR}
		\cR(x) = \cD_J\cdot x+\beta_J
	\end{align}
	we have
	\begin{align*}
		\int  \mathcal L(x,\cR(x))  \, \dd \mu(x) &\le  \int  \liminf\limits_{n \to \infty} \cL(x,\cR^{(n)}(x))  \, \dd \mu (x) \\
		&\le  \liminf\limits_{n \to \infty} \int  \cL(x,\cR^{(n)}(x))  \, \dd \mu(x)=\overline{\mathrm{err}}_{d}^{\cL}.
	\end{align*}

	We call a summand $j \in \{1, \dots, K\}$ degenerate if $A_j$ or $\overline {A}_j^c$ has $\mu$-measure zero.	
	Now, let $j$ be a non-degenerate summand. Since $\mu$ is nice there exists a $\partial A_j$-regular point $x$ that is not in $\bigcup_{A\in\IA:\partial A\not=\partial A_j} \partial A$, where $\IA := \{A_i: i \text{ is non-degenerate}\}$. 
	We let 
	$$
	J^x_-= \{i: x\in A_i\} \cup \{i: \overline A_i^c =A_j\} \text{ \ \ and \ \ } J^x_+=\{i: x\in A_i\}\cup\{i: A_i=A_j\}.
	$$ 
	Since $x\in \mathrm{supp}(\mu|_{\overline A_j^c})$ we get that the cell $A_{J^x_-}$ has strictly positive $\mu$-measure so that $J^x_-\in \IJ$. Analogously, 
	$x\in \mathrm{supp}(\mu|_{A_j})$ entails that $J^x_+\in\IJ$. (Note that $J_+^x$ and $J_-^x$ are just the cells that lie on the opposite sides of the hyperplane $\partial A_j$ at $x$.)
	We thus get that 
	$$
	\delta_{A_j}^{(n)}:= \sum_{i:A_i=A_j} \delta_i^{(n)} - \sum_{i:\overline{A}_i^c=A_j} \delta_i^{(n)} = \cD_{J_+^x}^{(n)}-\cD_{J_-^x}^{(n)}\to \cD_{J_+^x}-\cD_{J_-^x}=: \delta_{A_j},
	$$
	where the definitions of $\delta_{A_j}^{(n)}$ and $\delta_{A_j}$ do not depend on the choice of $x$. Analogously,
	$$
	\mathfrak b_{A_j}^{(n)}:= \sum_{i:A_i=A_j} \mathfrak b_{i}^{(n)}-\sum_{i:\overline{A}_i^c=A_j} \mathfrak b_{i}^{(n)} = \beta_{J_+^x}^{(n)}-\beta_{J_-^x}^{(n)} \to  \beta_{J_+^x}-\beta_{J_-^x}
	=: \mathfrak b_{A_j}.
	$$
We form the set $\IA_0$ by thinning  $\IA$ in such a way that for two active areas that share the same hyperplane as boundary only one is kept (meaning that for two active areas on opposite sides of a hyperplane only one is kept).
Then there exists an affine function $\mathfrak a:\R^{d_\mathrm{in}}\to\R$ such that for $x\in \bigcup_{J\in\IJ} A_J$
\begin{align}\label{eq82462}
\cR(x)= \mathfrak a(x) +\sum_{A\in\IA_0} \1_A(x) (\delta_A\cdot x+\mathfrak b_A).
\end{align}
We use the latter identity to define $\cR$ on the whole space $\R^{d_\mathrm{in}}$.

\underline{2. step:} Analyzing the multiplicities of  representation~(\ref{eq82462}). 

We assign the active areas in  $A\in\IA_0$ multiplicities. 
If
$$ \partial A \subset \{x\in\R^{d_\mathrm{in}} :\delta_A\cdot x+\mathfrak b_A=0\},$$
or, equivalently, $x\mapsto \1_A(x)(\delta_A\cdot x+\mathfrak b_A)$ is continuous,
then we assign $A$ the multiplicity one and otherwise two.

Next, we show that an active area $A\in\IA_0$ whose breakline is only served by one summand~$k$ of multiplicity one is again assigned multiplicity one. 
 For this it remains to show continuity of
    $x\mapsto \1_A(x)(\delta_A\cdot x+\mathfrak b_A)$ for such an $A$. Suppose that the $k$th summand is the unique summand that contributes to  $A$.   
    Then $\delta_k^{(n)}=\delta_A^{(n)}\to \delta_A$ and $\mathfrak b_k^{(n)}=\mathfrak b^{(n)}_A\to \mathfrak b_A$. Moreover, one has
    $$
    \{x\in\R^{d_\mathrm{in}}:\mathfrak n_k^{(n)}\cdot x- o_k^{(n)}=0\}\subset  \{x\in\R^{d_\mathrm{in}}:\delta _k^{(n)}\cdot x+ \mathfrak b_k^{(n)}=0\}
    $$
    which entails that, in particular, $\delta_k^{(n)}$ is a multiple of $\mathfrak n_k^{(n)}$. Both latter vectors converge and $|\mathfrak n_k|=1$ which also entails that the limit $\delta_A$ is a multiple of $\mathfrak n_k$.  To show that
    $$
    \partial A\subset \{x\in\R^{d_\mathrm{in}}:\delta _A\cdot x+ \mathfrak b_A=0\},
    $$
  it thus  suffices to  verify that one point of the hyperplane on the left-hand side lies also in the set on the right-hand side. Indeed, this is the case for  $x=o_k \vecn_k$ since for $n\to\infty$
 $$
  \delta_A\cdot (o_k\vecn_k)+\mathfrak b_A =\lim_{n\to\infty} \delta_k^{(n)}\cdot (o_k^{(n)}\vecn_k^{(n)}) +\mathfrak b_k^{(n)}=0.
 $$
 This entails contains continuity and we also showed that for such a $k$,  $\1_{A_k^{(n)}}(x)(\delta_k^{(n)}\cdot x+\mathfrak b_k^{(n)})\to \1_{A}(\delta_A\cdot x+\mathfrak b_A)$ pointwise in $x$.

Note that the sum over all  multiplicities assigned to the active areas $A\in\IA_0$ is strictly smaller than $d$ if
\begin{enumerate}
\item for all $n\in\N$, $m_0^{(n)}=1$,
\item there is a degenerate asymptotic active area, or
\item there is an active area $A\in\IA_0$ whose contributing terms have a cumulated multiplicity that is strictly larger  than the one assigned to $A$. 
\end{enumerate}
In the latter cases we can choose $m_0=1$ and $\cR$ as in~(\ref{eq82462}) is of dimension at most $d$.

\underline{3. step:} Separate treatment of the cases where the $\cR^{(n)}$ are of type (a), (b) or (c).

In case (a), we are done since property (1) above is satisfied.

Now suppose that the $\cR^{(n)}$ are all of type  (b) and let $\II\subset\{1,\dots,K\}$ denote the indices  of the summands  with multiplicity two. Then linear dependence of 
$(\vecn_j^{(n)}:j\in\II)$ implies linear dependence of the limits $(\vecn_j:j\in \II)$. If there is no degenerate asymptotic active area and the entries of $(\partial A_j:j\in\II)$ are pairwise different, then the representation~(\ref{eq82462}) satisfies property (b) and we verified that $\cR$ is in $\mathfrak R_d$. Conversely, in the case that the entries in  $(\partial A_j:j\in\II)$  are not pairwise different, then there is an active area $A\in\IA_0$ whose contributing summands contribute multiplicity at least four and  thus property (3) above holds and we are done.

It remains to consider the case (c). 
We can assume that there are no degenerate active areas and that every $A\in\IA_0$ is served by terms of total multiplicity at most two since otherwise  property (2) or (3) from above holds and we are done.
Then every $A\in\IA_0$ is served by 
\begin{enumerate}\item[(i)] a single summand,
\item[(ii)] two summands of multiplicity one that have the same asymptotic active area or
\item[(iii)] two summands of multiplicity one that have their asymptotic areas on opposite sides of the related hyperplane.
\end{enumerate}
For an $A\in \IA_0$ of type (i) or (ii) the asymptotic contribution of the related summands satisfies  outside the hyperplane $\partial A$ 
$$
\lim_{n\to\infty} \sum_{j:A_j=A} \1_{A_j^{(n)}}(x) (\delta_j^{(n)}\cdot x+\mathfrak b_j^{(n)})= \1_A(x)(\delta_A\cdot x+\mathfrak b_A).
$$
If there is a single summand of multiplicity one contributing the limit will be continuous for the same reason as above.  Hence, if there exists no $A \in \IA_0$ of type (iii) we use that $(\cR^{(n)})_{n \in \N}$ converges on $\bigcup_{J \in \IJ}A_J$ in order to deduce that $\lim_{n \to \infty} \mathfrak b^{(n)} =: \mathfrak b$ exists and obtain a representation (\ref{eq84529}) with $\mathfrak a \equiv \mathfrak b$ and are in case (c).

Now let $\IA_0^*$ denote the subset of all $A\in \IA_0$ that are of type (iii)  and assume that $\IA_0^* \neq \emptyset$.
 We say that $A\in \IA_0^*$ is served by the pair of twins $(i,j)$ if $A_i=A$ and $A_j=\overline A^c$. Moreover, we call a summand $k$ to be of type (iii) if it contributes to one active area of type (iii). By switching to an appropriate subsequence we can ensure that there exists a summand $k$ of type (iii) such that  $|\delta_{k}^{(n)}|$ is maximal over all summands  of type (iii) for all $n\in\N$.

We distinguish two cases. If there is a subsequence along which $(|\delta_k^{(n)}|)_{n \in \N}$ is uniformly bounded, then again by switching to appropriate subsequences we get that for every summand $j$ of type (iii), 
$\lim_{n\to\infty} \delta_j^{(n)}=:\delta_j\in\R^{d_\mathrm{in}}$ exists.
 Since $\1_{A_j^{(n)}(x)}(\delta_j^{(n)}\cdot x+\mathfrak b_j^{(n)})$ is continuous we get for a sequence $(x_n)_{n \in \N}$ that satisfies $x_n \in \partial A_j^{(n)}$, for all $n \in \N$, and that converges to an $x \in \partial A_j$ that
$$
\mathfrak b_j^{(n)}=-\delta_j^{(n)}\cdot x_n \to -\delta_j\cdot x =: \mathfrak b_j,
$$
where the left-hand side does not depend on the choice of $(x_n)_{n \in \N}$ or $x$. Moreover, the limit $\1_{A_j}(x)(\delta_j\cdot x+\mathfrak b_j)$ is continuous.
 Thus, for a pair of twins $(i,j)$  the contributing terms $i$ and $j$ to $A_i$ have in total multiplicity  two but the respective term 
	$$
	\1_{A_i}(x)(\delta_{A_i} \cdot x + \mathfrak b_{A_i})= \1_{A_i}(x)((\delta_i-\delta_j) \cdot x + \mathfrak b_i-\mathfrak b_j)
	$$
	is continuous and thus has multiplicity one. Hence, we are in case (3) above and are done.

It remains to consider the case where $(|\delta_k^{(n)}|)_{n \in \N}$ tends to infinity.
For every twin $(i,j)$ and $n\in\N$ we can choose $\alpha_i^{(n)} \in [-1,1]$ and $\alpha_j^{(n)}\in [-1,1]$ with
$$
\frac 1{|\delta_k^{(n)}|} \delta_i^{(n)} = \alpha_i^{(n)}\vecn_i^{(n)}\text{ \ and \ } \frac 1{|\delta_k^{(n)}|} \delta_j^{(n)} = - \alpha_j^{(n)}\vecn_j^{(n)}
$$
and along an appropriate subsequence we have convergence of all $(\alpha_i^{(n)})_{n \in \N}$ and  $(\alpha_j^{(n)})_{n \in \N}$ to limits $\alpha_i\in[-1,1]$ and $\alpha_j\in[-1,1]$.
Since $(\delta_i^{(n)}-\delta_j^{(n)})_{n\in\N}$ converges to $\delta_{A_i}$, $|\delta_k^{(n)}|\to\infty$ and $$\lim_{n\to\infty} \vecn_i^{(n)}= \vecn_i=-\vecn_j=\lim_{n\to\infty} -\vecn_j^{(n)}$$ we get that $\alpha_i=\alpha_j$. Consequently, for  $x\in \bigcup_{J\in\IJ} A_J$, one has
$$
\lim_{n\to\infty} \frac 1{|\delta_k^{(n)}|} (\cR^{(n)}(x) - \mathfrak b^{(n)}) = \sum_{i:A_i\in \IA_0^*}  \alpha_i (\vecn_i\cdot x-o_i).
$$
If $\sum_{i:A_i\in \IA_0^*} \alpha_i \vecn_i$ is not equal to zero, then the linear term on the right-hand side does not vanish. This contradicts convergence of $\cR^{(n)}$ on $\bigcup_{J\in\IJ}A_J$. Hence, we have $\sum_{i:A_i\in \IA_0^*} \alpha_i \vecn_i=0$. Note that not all $\alpha_i$'s are equal to zero since $\alpha_{k}\in\{\pm1\}$ and either $k$ or its twin appears in the sum.
Thus we showed that the normals belonging to the active areas in $\IA_0^*$ are linearly dependent. Hence, we are in case (b) and the proof is achieved.

\underline{4. step:} Discussion of the minimization  problem for strict responses at dimension $d$.

Now we choose a sequence of responses $(\cR^{(n)})_{n\in\N}$ from $\mathfrak R_d^\mathrm{strict}$ with
$$
\lim_{n\to\infty} \int \cL(x,\cR^{(n)}(x))\,\dd \mu(x)= \inf_{\tilde \cR\in\mathfrak R_d^\mathrm{strict}} \int \cL(x,\tilde\cR(x))\,\dd \mu(x).
$$
If infinitely many of the responses $\cR^{(n)}$ are in $\mathfrak R_{d-1}$ then
the response constructed above is a minimizer and  in $\mathfrak R_{d-1}$. Conversely, if all but finitely many responses are discontinuous, then the construction  above yields a limit $\cR\in\mathfrak R_d$ that minimizes the error.
Note that  there is at least one summand that contributes multiplicity two to one of the active areas of the limit $\cR$. If $\cR$ is continuous along the respective hyperplane, then it is of dimension at most $d-1$ and otherwise the response is discontinuous. In both cases we have $\cR\in\mathfrak R_d^\mathrm{strict}$. 
\end{proof}

\section{Strict generalized responses are not better than representable ones}
We will show that, in the setting of Theorem~\ref{thm:main2}, for every $d\in\{2,3,\dots\}$  with $\mathrm{err}_d^\cL<\mathrm{err}_{d-1}^\cL$, the infimum taken over the strict generalized responses produces a larger error than the best representable response. This will entail Theorem~\ref{thm:main2}. For a discussion on sufficient and necessary assumptions on the loss function $\cL$ that implies $\err_d^{\cL}<\err_{d-1}^{\cL}$ see Proposition~3.5 and Example~3.7 in \cite{DJK22}.

\begin{prop}\label{prop:minimal}
Suppose that the assumptions of Theorem~\ref{thm:main2} are satisfied.
Let $d\in\N_0$. Then there exists an optimal network $\IW\in\cW_d$ with 
$$
\mathrm{err}^\cL(\IW)=\mathrm{err}_d^\cL=\overline{\mathrm{err}}_d^\cL.
$$ 
If, additionally, $d\ge2$ and $\mathrm{err}^\cL_d<\mathrm{err}^\cL_{d-1}$,
then one has that
\begin{align}\label{eq8246}
\inf_{\cR \in \mathfrak R_d^{\mathrm{strict}}} \int \cL(x,\cR(x))\,\dd \mu(x)>\mathrm{err}_d^\cL.
\end{align}
\end{prop}

\begin{proof} We can assume without loss of generality that $\mu \neq 0$.
	First we verify  the assumptions of Proposition~\ref{prop:genres} in order to conclude that there are generalized responses $\cR\in\mathfrak R_d$ with 	$$
		\int \cL(x,\cR(x))\,\dd \mu(x)=\overline{\mathrm{err}}_d^\cL.
	$$
	We verify that  $\mu$ is a nice measure: In fact, since $\mu$ has Lebesgue density $h$, we have $\mu(H)=0$ for all hyperplanes $H \subset \R^{d_{\mathrm{in}}}$. Moreover, for every half-space $A$ with $\mu(A),\mu(\overline A^c)>0$ we have that $\partial A$ intersects the interior of the convex hull of $\ID$ so that there exists a point $x \in \partial A$ with $h(x)>0$. Since $\{x \in \R^{d_{\mathrm{in}}}:h(x)>0\}$ is an open set, $\{x \in \partial A: h(x)>0\}$ cannot be covered by finitely many hyperplanes different from $\partial A$. Moreover, since for all $x \in \R^{d_{\mathrm{in}}}$ the function $y \mapsto \cL(x,y)$ is strictly convex and attains its minimum we clearly have for fixed $x\in\R^{d_\mathrm{in}}$ continuity of $y \mapsto \cL(x,y)$ and 
	$$
	\lim\limits_{|y| \to \infty} \cL(x,y) =\infty.
	$$
	
 	We prove the  statement via induction over the dimension $d$.
		If $d\le 1$, all generalized responses of dimension $d$ are, on the compact set $\ID$, representable by a neural network and we are done. Now let $d\ge 2$ and suppose that $\cR$ is a \emph{strict} generalized response at dimension $d$ that satisfies
		$$
			\int \cL(x, \cR(x))\,\dd\mu(x)= \inf_{\tilde \cR \in \mathfrak R_d^{\mathrm{strict}}}  \int \cL(x, \tilde\cR(x))\,\dd\mu(x).
		$$
	It suffices to show that one of the following two cases enters: one has 
		\begin{align}\label{eq8362-1}
			\int\cL(x,\cR(x))\,\dd\mu(x) \ge \overline{\mathrm{err}}_{d-1}^\cL
		\end{align} or
		\begin{align}\label{eq8362-2}
			\int\cL(x,\cR(x))\,\dd\mu(x)> \overline{\mathrm{err}}_d^\cL.
		\end{align}
		Indeed, then in the  case that (\ref{eq8362-2}) does not hold, we have  as consequence of (\ref{eq8362-1}) that
		$$
		\overline{\mathrm{err}}_{d-1}^\cL \le \int\cL(x,\cR(x))\,\dd\mu(x) =\overline{\mathrm{err}}_d^\cL 
		$$
		and the induction hypothesis entails that  $\mathrm{err}_{d-1}^\cL=\overline{\mathrm{err}}_{d-1}^\cL \le \overline{\mathrm{err}}_{d}^\cL\le \mathrm{err}_{d}^\cL\le \mathrm{err}_{d-1}^\cL$ so that $ \mathrm{err}_{d}^\cL= \overline{\mathrm{err}}_{d}^\cL$ and $\mathrm{err}_{d}^\cL= \mathrm{err}_{d-1}^\cL$. 
		Thus, an optimal representable response $\cR$ of dimension at most  $d-1$ (which exists by induction hypothesis) is also optimal when taking the minimum over all generalized responses of dimension $d$ or smaller. 
		Conversely, if~(\ref{eq8362-2}) holds, an optimal generalized response (which exists by  Proposition~\ref{prop:genres}) is representable  so that, in particular, $\mathrm{err}_{d}^\cL= \overline{\mathrm{err}}_{d}^\cL$. 
		This shows that there always exists an optimal representable response.
		Moreover, it also follows that in the case where $\mathrm{err}_{d}^\cL< \mathrm{err}_{d-1}^\cL$, either of the properties~(\ref{eq8362-1}) and~(\ref{eq8362-2})  entail property (\ref{eq8246}).
	
		Suppose that the optimal strict generalized response $\cR$ at dimension $d$ is given in the standard representation  (\ref{eq84529}) 
			$$
			\cR(x)= \mathfrak a(x)+ \sum_{j=1}^K \1_{A_j}(x) \bigl(\delta_{j}\cdot x+\mathfrak b_{j}\bigr),
			$$
			with $A_1,\dots,A_K$ being the half-spaces with pairwise distinct boundaries and $m_0,m_1,\dots,m_K$ being the respective multiplicities. 
			Suppose that $\cR$ is an optimal response with the minimal number of terms of multiplicity two.
	
		First note that for every $k=1,\dots,K$ for which $\partial A_k$ does not intersect the interior of the convex hull of $\ID$, $A_k$ has either zero or full $\mu$-measure. In both cases we can remove the $k$th summand, set $m_0=1$ and adapt the affine function $\mathfrak a$ appropriately to get to a response that agrees $\mu$-almost everywhere with the former response and is again of dimension at most $d$. Thus we can without loss of generality assume that for every $k=1,\dots,K$, $\partial A_k$ intersects the interior of the convex hull of $\ID$.
			
		We distinguish cases. If one has $m_0=1$  (case (a)) or $\mathfrak a\equiv\mathfrak b\in\R$ (case (c)), then the proof of Proposition~3.3 in \cite{DJK22} shows that an appropriate replacement of a summand of multiplicity two by two summands of multiplicity one reduces the error which shows (\ref{eq8362-2}).
		
		It remains to treat the case (b). Let $\II$ be the set of indices with multiplicity two. Now we have that the vectors $(\vecn_j:j\in\II)$ are linear dependent. 
	If the set is not minimal in the sense that we can remove one of the vectors and still have a linearly dependent set, then we can argue as above and apply an appropriate replacement of this particular summand by two summands of multiplicity one that still satisfies (b) and has strictly smaller error. Hence, we can assume without loss of generality that the set $(\vecn_j:j\in\II)$
	is minimal in the sense that for a nontrivial linear combination 
	\begin{align}\label{eq246283}
	\sum_{j\in\II} \alpha_j \vecn_j=0
	\end{align}
	one has that $\alpha_j\not=0$ for all $j\in\II$. 
	For $x\in\ID$ and $y\in\R$ let
	 $$
	 \tilde \cL(x,y)=\cL\Bigl(x,y+\sum_{j\in\II^c}\1_{A_j}(x)( \delta_j \cdot x+\mathfrak b_j  )\Bigr)
	 $$
	 and note that due to continuity of $x\mapsto \sum_{j\in\II^c}\1_{A_j}(x)( \delta_j \cdot x+\mathfrak b_j  )$ the function $\tilde \cL$ satisfies the same assumptions  as imposed on $\cL$ in the proposition. If we can find
	  a representable response~$\hat\cR$  of dimension $2\#\II$ with
	$$
	\int \tilde \cL(x, \hat \cR(x))\,\dd\mu(x)< \int \tilde \cL(x, \tilde \cR(x))\,\dd\mu(x),
	$$
	where 
	$$
		\tilde \cR(x) = \mathfrak a(x) + \sum_{j \in \II} \1_{A_j}(x) (\delta_j \cdot x + \mathfrak b_j),
	$$
	 then
	 $$
	 \int  \cL\Bigl(x, \hat \cR(x)+\sum_{j\in \II^c} \1_{A_j}(x)( \delta_j \cdot x+\mathfrak b_j  )\Bigr)\,\dd\mu(x)< \int  \cL(x,  \cR(x))\,\dd\mu(x)
	 $$
	 and it follows validity of~(\ref{eq8362-2}).

	Therefore, we can assume without loss of generality that %$\cL=\tilde \cL$, 
	$\II=\{1,\dots,K\}$ and the optimal strict response is given by
	$$
	\cR(x)=\mathfrak a(x)+\sum_{j=1}^K \1_{A_j}(x)( \delta_j \cdot x+\mathfrak b_j).
	$$
	We fix a  non-trivial vector $(\alpha_j)_{j \in \II}$ satisfying~(\ref{eq246283}) and choose $(\delta_j^+,\delta_j^-) \in \R^{d_{\mathrm{in}}} \times \R^{d_{\mathrm{in}}}$, $(\mathfrak b_j^+,\mathfrak b_j^-)\in \R^2$ and $\mathfrak b\in\R$ with $\delta_j= \delta_j^+-\delta_j^-$ and $\mathfrak b_j=\mathfrak b_j^+-\mathfrak b_j^-$ such that for all $x \in \R^{d_\mathrm{in}}$
	%\ID \backslash (\bigcup_{j=1}^K \partial A_j)$
	$$
	\cR(x)= \mathfrak b+ \sum_{j=1}^K \1_{A_j}(x)(\delta_j^+\cdot x+\mathfrak b_j^+)+\1_{A_j^c}(x)(\delta_j^-\cdot x+\mathfrak b_j^-).
	$$
	By switching the sides of the active areas we can assume without loss of generality that all $\alpha_j>0$ (indeed this will change the definition of $\cR$ only on the respective hyperplanes).
	We will replace $\cN_j(x):=\1_{A_j}(x)(\delta_j^+\cdot x+\mathfrak b_j^+)+\1_{A_j^c}(x)(\delta_j^-\cdot x+\mathfrak b_j^-)$ by 
	$$
	\cN_j^\kappa(x):=(\delta_j^+\cdot x+\mathfrak b_j^++\kappa\alpha_j(\vecn_j\cdot x-o_j))^+-(-\delta_j^-\cdot x-\mathfrak b_j^--\kappa\alpha_j(\vecn_j\cdot x-o_j))^+,
	$$
	where $\kappa>0$. 
	Let
	$$
		Q_j^\kappa := \{x \in \R^{d_{\mathrm{in}}}: \cN_j^\kappa(x) \neq \cN_j(x)+\kappa \alpha_j (\vecn_j \cdot x - o_j)\}
	$$
	and compare $\cR$ with
	$$
	\cR^\kappa(x):=\mathfrak b^\kappa+\sum_{j=1}^K \cN_j^\kappa(x),
	$$
	 where $\mathfrak b^{\kappa} := \mathfrak b + \sum_{j=1}^K \kappa \alpha_j o_j$. Note that since $\sum_{j=1}^K \alpha_j \vecn_j =0$ we have $\cR^\kappa=\cR$, on $\R^{d_{\mathrm{in}}} \backslash (\bigcap_{j=1, \dots, K} Q_j^\kappa)$.
	Furthermore, the set $\{x \in \R^{d_{\mathrm{in}}}:x\text{ is in two }Q_j^\kappa\}\cap\ID$ is of size $\cO(\kappa^{-2})$. Hence, 
	\begin{align}\label{eq935621}
		\int \cL(x,\cR^\kappa(x))-\cL(x,\cR(x)) \,\dd\mu(x)=\sum_{j=1}^K \int_{ Q_j^\kappa} h(x) ( \cL(x,\cR^\kappa(x))-\cL(x,\cR(x))) \, \dd x + \cO(\kappa^{-2}).
	\end{align}
	Moreover, using the uniform continuity of $h$ on the compact set $\ID$ and the uniform boundedness of $|\cL(x,\cR^\kappa(x))-\cL(x,\cR(x)))|$ over all $x\in\ID$ and $\kappa\ge1$ we conclude that as $\kappa\to\infty$
	\begin{align}\begin{split}\label{eq73461}
	\int_{ Q_j^\kappa} & h(x) ( \cL(x,\cR^\kappa(x))-\cL(x,\cR(x))) \, \dd x\\
	&=\int _{H_j} \int_{  (x'+\R\vecn_j)\cap  Q_j^\kappa}  h(z)\bigl ( \cL(z,\cR^\kappa(z))-\cL(z,\cR(z))\bigr)\,\dd z \, \dd x'\\
	&=\int _{H_j} h(x') \int_{  (x'+\R\vecn_j)\cap  Q_j^\kappa}  \bigl ( \cL(x',\cR^\kappa(z))-\cL(x',\cR(z))\bigr)\,\dd z \, \dd x'+o(\kappa^{-1}),
	\end{split}\end{align}
	where $H_j = \partial A_j$. 
	Now note that for a fixed $x'\in H_j$ for which $(x'+\vecn_j \R) \cap  Q_j^\kappa$ does not intersect one of the $Q_i^\kappa$ with $i\not =j$, one has 
		\begin{align*}
	\int_{(x'+\R\vecn_j ) \cap  Q_j^\kappa}  \cL(x',\cR(z))\,\dd z =&|(x'+\R\vecn_j)\cap Q_j^\kappa \cap A_j|\,  (L^+_j(x')+o(1)) \\
	&+|(x'+\R\vecn_j)\cap Q_j^\kappa \cap A^c_j|\,  (L^-_j(x')+o(1)),
	\end{align*}
	where $|\cdot|$ denotes the $1$-dimensional Hausdorff measure (i.e., the length of the segment),
	$$
	L_j^+(x')=\cL(x', \delta_j^+\cdot x'+\mathfrak b_j^+ +\hat \cR_j (x') ) \text{ \ and \ } 
	L_j^-(x')= \cL(x', \delta_j^-\cdot x'+\mathfrak b_j^- +\hat \cR_j (x') )
	$$
	with
	$$
		\hat \cR_j (x')=\mathfrak b+ \sum_{i \neq j} \bigl(\1_{A_i}(x')(\delta_i^+\cdot x'+\mathfrak b_i^+)+\1_{A_i^c}(x')(\delta_i^-\cdot x'+\mathfrak b_i^-)\bigr).
	$$
	Moreover, for the same $x'$
	$$
	\int_{(x'+\R \vecn_j) \cap  Q_j^\kappa}  \cL(x',\cR^\kappa(z))\,\dd z  = |(x'+\vecn_j \R) \cap  Q_j^\kappa| \,  (\overline L_j(x')+o(1)),
	$$
	where $\overline L_j(x')$ is the average of $\cL(x',\cdot)$ on the segment  $[\delta_j^-\cdot x'+\mathfrak b_j^-+\hat \cR_j(x'), \delta_j^+\cdot x'+\mathfrak b_j^++\hat \cR_j(x') ]$. 
		
	We calculate the Hausdorff measure of the segments $(x'+\R\vecn_j)\cap Q_j^\kappa \cap A_j$ and $(x'+\R\vecn_j)\cap Q_j^\kappa \cap A_j^c$.
		We note that for $t\in\R$, $x'+t\vecn_j$ lies in $Q_j^\kappa$ if $t$ lies between the solutions $t_{+/-}^\kappa$ of
		$$
		 \delta_j^{+/-} \cdot ( x'+t^\kappa_{+/-}\vecn_j)+\mathfrak b_j^{+/-}+ \kappa \alpha_j t_{+/-}^\kappa=0 
				$$
	so that $|(x'+\R\vecn_j)\cap Q_j^\kappa \cap A_j|=|[t_-^\kappa,t_+^\kappa]\cap[0,\infty)|$.  Since 
		$$
		\lim_{\kappa\to\infty} \kappa\, t^\kappa_{+/-}=-\frac 1{\alpha_j}(\delta_j^{+/-}\cdot x'+\mathfrak b_j^{+/-})=:t_{+/-}(x')
		$$
we get that 
		$$
		q_j^+(x'):=\lim_{\kappa\to \infty} \kappa \,\bigl|(x'+\R\vecn_j)\cap Q_j^\kappa \cap A_j\bigr|= \bigl| [t_-(x'),t_+(x')]\cap[0,\infty)\bigr|
		$$
Analogously, it follows that
		$$
		q_j^-(x'):=\lim_{\kappa\to \infty} \kappa \,\bigl|(x'+\R\vecn_j)\cap Q_j^\kappa \cap A_j^c\bigr|= \bigl| [t_-(x'),t_+(x')]\cap(-\infty,0]\bigr|.
		$$
	Combining the estimates gives that for every $x'\in H_j\backslash \bigcup_{i\not=j} H_i$ one has
	\begin{align*}
	\lim_{\kappa\to\infty} \kappa &\int_{  (x'+\R\vecn_j)\cap  Q_j^\kappa}  \bigl ( \cL(x',\cR^\kappa(z))-\cL(x',\cR(z))\bigr)\,\dd z\\
	&=(q_j^+(x')+q_j^-(x')) \overline L_j(x') - (q_j^-(x')\,L^+_j(x')  +q_j^+(x')\,    L^-_j(x'))
	\end{align*}

With~(\ref{eq935621}), (\ref{eq73461}) and dominated convergence  we get that
	\begin{align*}
	\lim_{\kappa\to\infty}	&\kappa \int \bigl(\cL(x,\cR^\kappa(x))-\cL(x,\cR(x)) \bigr)\,\dd\mu(x) \\
		&=\sum_{j=1}^K \int_{ H_j}  h(x') \bigl((q_j^+(x')+q_j^-(x'))\, \overline L_j(x') - (q_j^+(x')\,L^+_j(x')  +q_j^-(x')\,    L^-_j(x'))\bigr) \,\dd x' , 
	\end{align*}	
where we used that $\kappa  \int_{  (x'+\R\vecn_j)\cap  Q_j^\kappa}  \bigl ( \cL(x',\cR^\kappa(z))-\cL(x',\cR(z))\bigr)\,\dd z$ is uniformly bounded over all $j=1,\dots K$, $x'\in H_j\cap \ID$ and $\kappa\ge1$.

Now consider $\cR^{-\kappa}$ given by 
	$$
		\cR^{-\kappa}(x)=\mathfrak b^{-\kappa} +\sum_{j=1}^K -(-\delta_j^+\cdot 	x-\mathfrak b_j^++\kappa\alpha_j(\vecn_j\cdot x-o_j))^+ +(\delta_j^-\cdot x+\mathfrak b_j^--\kappa\alpha_j(\vecn_j\cdot x-o_j))^+.	
	$$
	where $\mathfrak b^{-\kappa} := \mathfrak b - \sum_{j=1}^K \kappa \alpha_j o_j$.
	Following the same arguments as above we get that
	\begin{align*}
	\lim_{\kappa\to\infty}	&\kappa\int \cL(x,\cR^{-\kappa}(x))-\cL(x,\cR(x)) \,\dd\mu(x) \\
		&=\sum_{j=1}^K \int_{ H_j}  h(x') \bigl((q_j^+(x')+q_j^-(x')) \overline L_j(x') - (q_j^-(x')\,L^+_j(x')  +q_j^+(x')\,    L^-_j(x'))\bigr) \,\dd x'  .
	\end{align*}	
	Adding the estimates we get with $q_j(x')=q_j^+(x')+q_j^-(x')$ that 
	\begin{align}\begin{split}\label{eqw426}
	\lim_{\kappa\to\infty}& \kappa \int \bigl(\cL(x,\cR^{-\kappa}(x))+\cL(x,\cR^{\kappa}(x))-2\cL(x,\cR(x))\bigr) \,\dd\mu(x)\\
	&=\sum_{j=1}^K \int_{ H_j}  h(x') q_j(x')  \bigl(2 \overline L_j(x') - (L^+_j(x')  +   L^-_j(x'))\bigr) \,\dd x'.
	\end{split}\end{align}
By strict convexity of $\cL(x',\cdot)$, one has $ 2 \overline L_j(x') \le L_j^+(x')+L_j^-(x') $ with strict inequality whenever $\delta_j^- \cdot x'+\mathfrak b_j^- \neq \delta_j^+ \cdot x' + \mathfrak b_j^+$.
	Since $\partial A_1 \not \subset \{x \in \R^{d_{\mathrm{in}}}:\delta_1 \cdot x + \mathfrak b_1 =0\}$ we have that the set $H_1'$ consisting of all $x' \in H_1$ such that $h(x')>0$, $q_1(x')>0$ and $\delta_j^- \cdot x'+\mathfrak b_j^- \neq \delta_j^+ \cdot x' + \mathfrak b_j^+$
	has strictly positive $(d_{\mathrm{in}}-1)$-dimensional Hausdorff measure. 
	Consequently, the limit in~(\ref{eqw426}) is strictly negative and there exists $\kappa>0$ for which either $\cR^\kappa$ or $\cR^{-\kappa}$ is a better response than $\cR$ which contradicts optimality of $\cR$.
	\end{proof}

{\bf Acknowledgments.}
Funded by the Deutsche Forschungsgemeinschaft (DFG, German Research Foundation) under Germany's Excellence Strategy EXC 2044--390685587, Mathematics Münster: Dynamics--Geometry--Structure. Funded by the Deutsche Forschungsgemeinschaft (DFG, German Research Foundation) – SFB 1283/2 2021 – 317210226.

\bibliographystyle{alpha}
\bibliography{Optimal_networks}

\begin{thebibliography}{EMWW20}

\bibitem[AB09]{attouch2009convergence}
H.~Attouch and J.~Bolte.
\newblock On the convergence of the proximal algorithm for nonsmooth functions
  involving analytic features.
\newblock {\em Mathematical Programming}, 116(1):5--16, 2009.

\bibitem[AMA05]{absil2005convergence}
P.~A. Absil, R.~Mahony, and B.~Andrews.
\newblock Convergence of the iterates of descent methods for analytic cost
  functions.
\newblock {\em SIAM Journal on Optimization}, 16(2):531--547, 2005.

\bibitem[BDL07]{bolte2007lojasiewicz}
J.~Bolte, A.~Daniilidis, and A.~Lewis.
\newblock The \uppercase{{\L}}ojasiewicz inequality for nonsmooth subanalytic
  functions with applications to subgradient dynamical systems.
\newblock {\em SIAM Journal on Optimization}, 17(4):1205--1223, 2007.

\bibitem[CK22]{christof2022omnipresence}
C.~Christof and J.~Kowalczyk.
\newblock On the omnipresence of spurious local minima in certain neural
  network training problems.
\newblock arXiv:2202.12262, 2022.

\bibitem[DDKL20]{MR4056927}
D.~Davis, D.~Drusvyatskiy, S.~Kakade, and J.~D. Lee.
\newblock Stochastic subgradient method converges on tame functions.
\newblock {\em Foundations of Computational Mathematics}, 20(1):119--154, 2020.

\bibitem[DJK23]{DJK22}
S.~Dereich, A.~Jentzen, and S.~Kassing.
\newblock On the existence of minimizers in shallow residual
  \uppercase{R}e\uppercase{L}\uppercase{U} neural network optimization
  landscapes.
\newblock arXiv:2302.14690, 2023.

\bibitem[DK21]{dereich2021convergence}
S.~Dereich and S.~Kassing.
\newblock Convergence of stochastic gradient descent schemes for
  \uppercase{{\L}}ojasiewicz-landscapes.
\newblock arXiv:2102.09385, 2021.

\bibitem[DK22]{dereich2021cooling}
S.~Dereich and S.~Kassing.
\newblock Cooling down stochastic differential equations: Almost sure
  convergence.
\newblock {\em Stochastic Processes and their Applications}, 152:289--311,
  2022.

\bibitem[EJRW21]{Eberle2021arXiv}
S.~Eberle, A.~Jentzen, A.~Riekert, and G.~S. Weiss.
\newblock Existence, uniqueness, and convergence rates for gradient flows in
  the training of artificial neural networks with
  \uppercase{R}e\uppercase{L}\uppercase{U} activation.
\newblock arXiv:2108.08106, 2021.

\bibitem[EMWW20]{E2020towards}
W.~E, C.~Ma, S.~Wojtowytsch, and L.~Wu.
\newblock Towards a mathematical understanding of neural network-based machine
  learning: what we know and what we don't.
\newblock arXiv:2009.10713, 2020.

\bibitem[GJL22]{gallonjentzenlindner2022blowup}
D.~Gallon, A.~Jentzen, and F.~Lindner.
\newblock Blow up phenomena for gradient descent optimization methods in the
  training of artificial neural networks.
\newblock arXiv:2211.15641, 2022.

\bibitem[GP90]{girosi1990networks}
F.~Girosi and T.~Poggio.
\newblock Networks and the best approximation property.
\newblock {\em Biological Cybernetics}, 63(3):169--176, 1990.

\bibitem[JR22]{jentzen2021existence}
A.~Jentzen and A.~Riekert.
\newblock On the existence of global minima and convergence analyses for
  gradient descent methods in the training of deep neural networks.
\newblock {\em Journal of Machine Learning}, 1(2):141--246, 2022.

\bibitem[KKV03]{kainen2003best}
P.~C. Kainen, V.~Kurkov\'{a}, and A.~Vogt.
\newblock Best approximation by linear combinations of characteristic functions
  of half-spaces.
\newblock {\em Journal of Approximation Theory}, 122(2):151--159, 2003.

\bibitem[LMQ22]{lim2022best}
L.-H. Lim, M.~Micha{\l}ek, and Y.~Qi.
\newblock Best k-layer neural network approximations.
\newblock {\em Constructive Approximation}, 55(1):583--604, 2022.

\bibitem[{\uppercase{\L}}oj63]{lojasiewicz1963propriete}
S.~{\uppercase{\L}}ojasiewicz.
\newblock Une propri{\'e}t{\'e} topologique des sous-ensembles analytiques
  r{\'e}els.
\newblock {\em Les {\'e}quations aux d{\'e}riv{\'e}es partielles}, 117:87--89,
  1963.

\bibitem[{\uppercase{\L}}oj65]{lojasiewicz1965ensembles}
S.~{\uppercase{\L}}ojasiewicz.
\newblock Ensembles semi-analytiques.
\newblock {\em Lectures Notes IHES (Bures-sur-Yvette)}, 1965.

\bibitem[{\uppercase{\L}}oj84]{lojasiewicz1984trajectoires}
S.~{\uppercase{\L}}ojasiewicz.
\newblock Sur les trajectoires du gradient d’une fonction analytique.
\newblock {\em Seminari di geometria}, 1982/1983:115--117, 1984.

\bibitem[MHKC20]{mertikopoulos2020sure}
P.~Mertikopoulos, N.~Hallak, A.~Kavis, and V.~Cevher.
\newblock On the almost sure convergence of stochastic gradient descent in
  non-convex problems.
\newblock In {\em Advances in Neural Information Processing Systems},
  volume~33, pages 1117--1128, 2020.

\bibitem[PRV21]{petersen2021topological}
P.~Petersen, M.~Raslan, and F.~Voigtlaender.
\newblock Topological properties of the set of functions generated by neural
  networks of fixed size.
\newblock {\em Foundations of Computational Mathematics}, 21(2):375--444, 2021.

\bibitem[SCP16]{swirszcz2016local}
G.~Swirszcz, W.~M. Czarnecki, and R.~Pascanu.
\newblock Local minima in training of neural networks.
\newblock arXiv:1611.06310, 2016.

\bibitem[Sin70]{singerbest}
I.~Singer.
\newblock {\em Best Approximation in Normed Linear Spaces by Elements of Linear
  Subspaces}.
\newblock Springer, 1970.

\bibitem[SS18]{safran2018spurious}
I.~Safran and O.~Shamir.
\newblock Spurious local minima are common in two-layer
  \uppercase{R}e\uppercase{L}\uppercase{U} neural networks.
\newblock In {\em International conference on machine learning}, pages
  4433--4441. PMLR, 2018.

\bibitem[Tad15]{TADIC2015convergence}
V.~B. Tadić.
\newblock Convergence and convergence rate of stochastic gradient search in the
  case of multiple and non-isolated extrema.
\newblock {\em Stochastic Processes and their Applications}, 125(5):1715--1755,
  2015.

\bibitem[VBB19]{Venturi2019Spurious}
L.~Venturi, A.~S. Bandeira, and J.~Bruna.
\newblock Spurious valleys in one-hidden-layer neural network optimization
  landscapes.
\newblock {\em Journal of Machine Learning Research}, 20(133):1--34, 2019.

\end{thebibliography}

\end{document}